\documentclass{article}

\usepackage[utf8]{inputenc}
\usepackage[T1]{fontenc}
\usepackage{microtype}
\usepackage{algorithm}
\usepackage[noend]{algpseudocode}
\usepackage{amsmath, amsfonts, amssymb, amsthm}
\usepackage{upgreek}
\usepackage{graphicx}
\usepackage{subfigure}
\usepackage{thmtools}
\usepackage{thm-restate}
\usepackage{hyperref}  
\usepackage[capitalize, nameinlink]{cleveref}
\usepackage{comment}
\usepackage{enumerate}
\usepackage{mathtools}
\usepackage[dvipsnames]{xcolor}
\usepackage{nag}
\usepackage{stmaryrd}
\usepackage{parskip}
\usepackage{upgreek}
\usepackage[numbers, compress]{natbib}
\usepackage[skins]{tcolorbox}
\usepackage{xspace}
\usepackage{xfrac}
\usepackage{float}
\usepackage{dblfloatfix}

\usepackage{url}            
\usepackage{booktabs}       
\usepackage{nicefrac}       

\graphicspath{{figures/}{.}}

\newcommand{\declarecolor}[2]{\definecolor{#1}{RGB}{#2}\expandafter\newcommand\csname #1\endcsname[1]{\textcolor{#1}{##1}}}
\definecolor{Green}{rgb}{0.05, 0.35, 0.01}

\crefformat{equation}{(#2#1#3)}

\newcommand{\R}{\mathbb{R}}

\renewcommand{\emptyset}{\varnothing}

\DeclareMathOperator*{\argmin}{arg\,min\,}
\renewcommand{\P}{\textnormal{P}}
\newcommand{\NP}{\textnormal{NP}}
\newcommand{\MTPP}{\textnormal{\textsf{MTPP}}\xspace}
\newcommand{\cost}{\textnormal{\texttt{work}}\xspace}
\newcommand{\size}{\ensuremath{\texttt{size}_\texttt{param}}\xspace}
\newcommand{\outputsize}{\ensuremath{\texttt{size}_\texttt{out}}\xspace}
\newcommand{\iocost}{\textnormal{\texttt{io}\xspace}}
\newcommand{\overflowcost}{\textnormal{\texttt{overflow}}\xspace}
\newcommand{\OPT}{\textnormal{OPT}\xspace}

\newcommand{\mat}[1]{\mathbf{#1}}
\newcommand{\outnbr}{\ensuremath{N^+}}
\newcommand{\innbr}{\ensuremath{N^-}}
\newcommand{\ie}{i.e.,\xspace}
\newcommand{\eg}{e.g.,\xspace}

\newcommand{\peakmem}{\texttt{peak}\xspace}
\newcommand{\blockmem}{M\xspace}
\newcommand{\khat}{\ensuremath{k'}}
\newcommand{\nhat}{\ensuremath{r}}
\newcommand{\bhat}{\ensuremath{b'}}
\newcommand{\MyComment}[1]{\texttt{\textcolor{Green}{// #1}}}

\newcommand{\SegmentCostDataStructure}{\textnormal{\texttt{SegmentCostDataStructure}}\xspace}
\newcommand{\Initialize}{\textnormal{\texttt{Initialize}}\xspace}
\newcommand{\Query}{\textnormal{\texttt{Query}}\xspace}
\newcommand{\SliceGraph}{\textnormal{\texttt{SliceGraph}}\xspace}
\newcommand{\Solve}{\textnormal{\texttt{DP}}\xspace}
\newcommand{\Kahn}{\textnormal{\texttt{KahnWithNodePriorities}}\xspace}

\newcommand{\blockcost}{\textnormal{\texttt{block}}\xspace}
\newcommand{\simple}{\textnormal{\texttt{simple}}\xspace}
\newcommand{\bottleneck}{\textnormal{\texttt{bottleneck}}\xspace}
\newcommand{\bottleguess}{\textnormal{\texttt{bottleneck-guess}}\xspace}
\newcommand{\exact}{\textnormal{\texttt{exact}}\xspace}

 \newcommand{\cP}{\mathcal{P}}

\DeclarePairedDelimiter{\abs}{\lvert}{\rvert}
\DeclarePairedDelimiter{\set}{\{}{\}}
\DeclarePairedDelimiter{\parens}{(}{)}

\theoremstyle{plain}
\newtheorem{theorem}{Theorem}[section]
\newtheorem{lemma}[theorem]{Lemma}

\newtheorem{warmup}[theorem]{Warmup}

\theoremstyle{definition}
\newtheorem{definition}[theorem]{Definition}

\newtheorem{remark}[theorem]{Remark}

\usepackage[accepted]{icml2024}
\usepackage{inconsolata}

\icmltitlerunning{Practical Performance Guarantees for Pipelined DNN Inference}

\begin{document}

\twocolumn[
\icmltitle{Practical Performance Guarantees for Pipelined DNN Inference}

\icmlsetsymbol{equal}{*}

\begin{icmlauthorlist}
\icmlauthor{Aaron Archer}{equal,google}
\icmlauthor{Matthew Fahrbach}{equal,google}
\icmlauthor{Kuikui Liu}{mit}
\icmlauthor{Prakash Prabhu}{google}
\end{icmlauthorlist}

\icmlaffiliation{google}{Google}
\icmlaffiliation{mit}{MIT}

\icmlcorrespondingauthor{Matthew Fahrbach}{fahrbach@google.com}

\icmlkeywords{Machine Learning, ICML}

\vskip 0.3in
]

\printAffiliationsAndNotice{\icmlEqualContribution} 

\begin{abstract}
We optimize pipeline parallelism for deep neural network (DNN) 
inference by partitioning model graphs into $k$ stages and 
minimizing the running time of the bottleneck stage, including 
communication.
We give practical and  effective algorithms 
for this NP-hard problem, but our emphasis is on tackling the 
practitioner's dilemma of deciding when a solution is good enough.
To this end, we design novel mixed-integer programming (MIP)
relaxations for proving lower bounds. Applying these methods to 
a diverse testbed of 369 production models, for $k \in \{2, 4, 8, 16, 32, 64\}$,
we empirically show that 
these lower bounds are strong enough to be useful in practice.
Our lower bounds are substantially stronger than standard combinatorial bounds.
For example, evaluated via 
geometric means across a production testbed with $k = 16$ pipeline 
stages, our MIP formulations raise the lower bound from 0.4598 
to 0.9452, expressed as a fraction of the best partition found.
In other words, our improved lower bounds close the optimality gap by a factor of 9.855x.
\end{abstract}

\section{Introduction}
\label{sec:introduction}

Large-scale machine learning (ML) workloads rely on distributed systems and specialized hardware accelerators, \eg graphics 
processing units (GPUs) and tensor processing units (TPUs).
Fully utilizing this hardware, however,
remains an increasingly important challenge.
ML accelerators have a small amount of \emph{fast memory} co-located with each 
computational unit (CU),
and a much larger amount of \emph{slow memory} that is accessed via an interconnect shared among the CUs.
Achieving peak performance for deep neural network (DNN) training and inference
requires the ML compiler and/or practitioner
to pay significant attention to where intermediate
data is stored and how it flows between CUs.
This work addresses pipeline partitioning for DNNs to maximize 
\emph{inference throughput}, with a particular focus on lower 
bound methods for proving per-instance approximation ratios.

ML inference handles two main types of data:
\emph{model parameters} (weights learned during training) and \emph{activations}
(intermediate outputs of the model, \eg from hidden layers).
Keeping activations in fast memory (\eg SRAM) is critical, so ML compilers
often treat it as a hard constraint.
Model parameters can be streamed from slow memory,
but caching them in fast memory greatly boosts performance.
When we partition an end-to-end inference computation into a linear pipeline 
with $k$ stages and process each stage on a different CU,
we \emph{increase the amount of fast memory} at our disposal by a factor of~$k$,
allowing us to cache more parameters
and support larger activations
(and hence larger models and/or batch sizes).
However, this introduces two main challenges:
(1) CUs must send their outputs downstream,
often via a slow and contended data channel,
so we need to minimize communication overhead;
and (2) we must balance the running time across all stages of the partition
since the \emph{overall throughput is governed by the bottleneck stage}.

\begin{figure}[t]
\centering
\vspace{-0.40cm}
\includegraphics[width=0.48\textwidth]{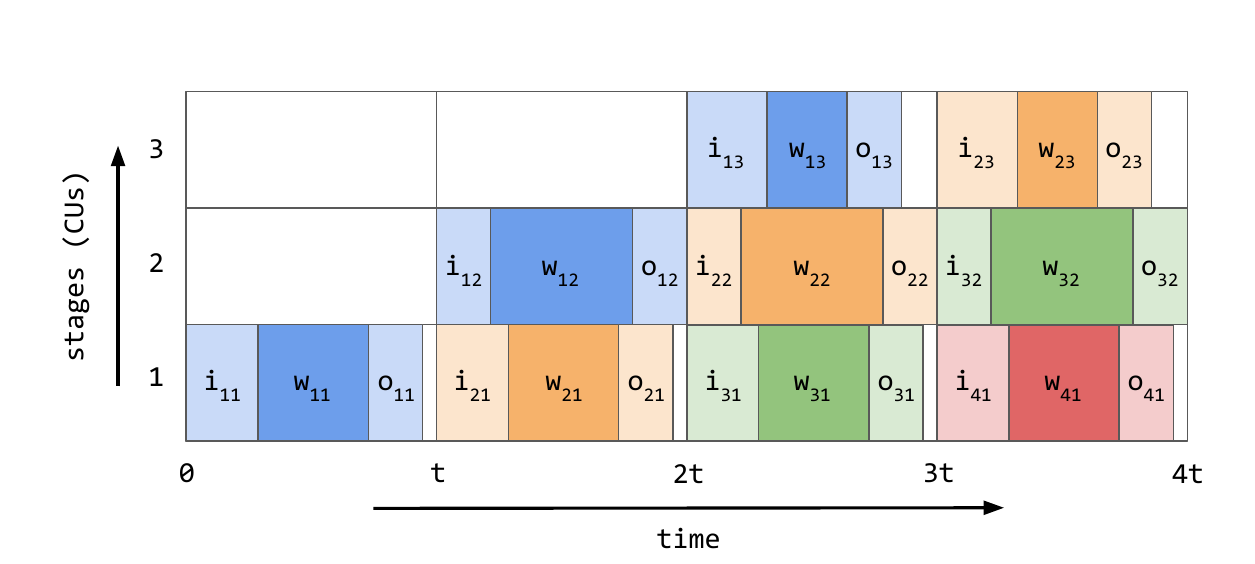}
\vspace{-0.85cm}
\caption{Inference pipeline from startup to steady state with $k=3$ stages.
Each inference batch is represented with the same color as it advances through the pipeline.
Values $\texttt{i}_{b \ell}$, $\texttt{w}_{b \ell}$, $\texttt{o}_{b \ell}$
are the times needed for stage $\ell$ to get its input for batch $b$,
process it, and flush its output.
Stage~$2$ is the bottleneck, \ie $t = \texttt{i}_{*2} + \texttt{w}_{*2} + \texttt{o}_{*2}$,
and limits system throughput.
Empty space (white) denotes idle time.
}
\vspace{-0.5cm}
\label{fig:pipeline-example}
\end{figure}

\subsection{Practioner's dilemma}
\label{sec:dilemma}

Suppose you are an ML engineer who has been tasked with partitioning a model graph for pipelined inference as illustrated in \Cref{fig:pipeline-example}.
We explain the practitioner's dilemma with a toy example.
Assume you are searching for the best way to partition a model among 8 processors, and you discover a solution where the bottleneck processor takes 10ms to finish. In this case, the pipeline finishes one inference every 10ms, so the throughput is 1 inference / 10ms = 100 inferences per second, and the latency of a single inference is 8 x 10ms = 80ms. Is this a good solution? How do you know there isn't one that is ten times better?

Suppose now that you devise an approximation algorithm and prove it has a worst-case \emph{approximation ratio} of 2. Congratulations, proving a worst-case bound is often no easy feat! You run your algorithm on the same instance as before, generating a partition with a bottleneck stage of 12ms, and a \emph{lower bound} of 6ms. Is this a good result? What if your boss tells you that your company is about to spend millions of dollars on hardware to run your inference pipeline. If you can improve your 12ms solution all the way down to 6ms, then you can save half of this hardware or run twice as many more inferences.
This is great motivation to improve your solution, but how do you know when to stop?
\emph{It could be that the lower bound is weak, not your solution.}

Two things went wrong here. Your algorithm and your lower bound are robust to all inputs, but you have a particular instance in front of you, and that is all you care about. If you could run a different heuristic to output a solution with a bottleneck stage of 10ms, along with a lower bound certificate for this instance of 9.5ms, that means it is impossible to improve the solution by more than 5\%. Presumably, this will make both you and your boss happier, and suggest that you can now spend your time on something else.

In practice, we care about good performance on a whole family of instances. Partitioning algorithms run inside ML compilers,
often with tight time constraints that preclude computing strong lower bounds in the compiler itself. In this case, one way to gain confidence in the quality of partitioning algorithms is to create a testbed of instances that are representative of the ones we solve in practice,
run our algorithms to generate solutions and lower bounds offline, and examine the approximation ratios we were able to prove for these instances.
If the average per-instance approximation ratio is 1.05 on the testbed, we argue this should give us more confidence in the partitioning algorithm than would the proof that a different algorithm has worst-case ratio 2. If our lower bounds are fast enough to run within the compiler's time limits, that is even better, as we can then bound the suboptimality of each instance, rather than trusting that the testbed results generalize to the instance at hand.

This ethos motivates the focus of our paper: we describe a sequence of successively stronger (but costlier to compute) lower bound methods that can be used to prove per-instance lower bounds for the pipeline partitioning problem.
We also give algorithms for constructing partitions for the original problem, and show that the cost of these partitions is close to the lower bounds
across a testbed of hundreds of production models with a variety of architectures.

\subsection{Our contributions and techniques}
The main contributions of this work are as follows:
\begin{enumerate}
    \item We formalize the \emph{max-throughput partitioning problem} (\MTPP) for pipelined inference,
    and we prove that it is NP-hard.
    We then formulate a novel mixed-integer program (MIP) for \MTPP,
    and study sparse relaxations
    to obtain strong lower bounds (\Cref{sec:lower_bounds}).
    
    \item We give a fast and practical pipeline partitioning algorithm called \SliceGraph
    that combines dynamic programming with a biased random-key genetic algorithm (\Cref{sec:algorithm}).

    \item We present extensive experiments across real and synthetic
    model graphs for a wide variety of ML architectures and workloads (\Cref{sec:experiments}).
    Using our (\emph{a posteriori}) MIP lower bounds, we demonstrate 
    that \SliceGraph is highly effective in practice,
    \eg for $k \leq 16$, our strongest lower bound is (on average) 95.5\% of the \SliceGraph solution, whereas the 
    standard combinatorial lower bound is only 46.0\%. 
\end{enumerate}

\section{Preliminaries}
\label{sec:preliminaries}

To further build intuition,
it is helpful to think of pipelined inference as an assembly line
where the model is split into $k$ stages and the inputs for each stage
are produced earlier in the assembly line.
If $t$ is the running time of the longest stage (\ie the \emph{bottleneck}),
each stage can finish its local computation in parallel in time $t$.
Every~$t$ units of time
we advance each batch one stage forward in the assembly line (see \Cref{fig:pipeline-example}),
so the end-to-end latency for a batch of inferences is $kt$
and the throughput of the system is $\sfrac{1}{t}$ (\ie one batch per $t$ units of time).
To maximize system throughput, it is critical to the partition work in a way
that minimizes the bottleneck time $t$.

\begin{figure*}
\centering
\vspace{0.1cm}
\includegraphics[width=0.95\textwidth]{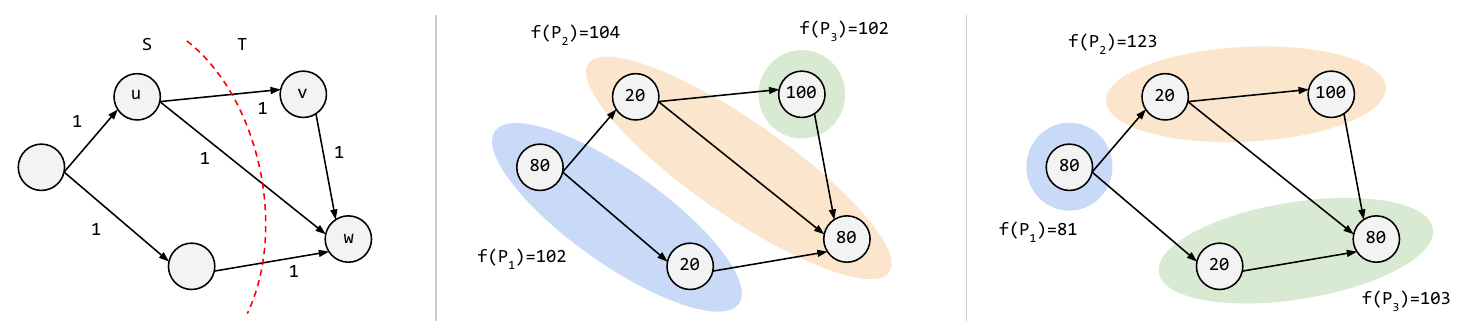}
\vspace{-0.1cm}
\caption{Partitioning computation graphs:
(left) tensor cut property where $\iocost(S, T) = 2$ because $v$ and $w$ consume the same tensor;
(middle) invalid partition because blocks $P_2$ and $P_3$ form a cycle in the quotient graph;
(right) valid partition with block costs for $k=3$.}
\label{fig:graph_partitions}
\end{figure*}

\subsection{Computation graphs}

An ML model is commonly represented as a
computation graph $G = (V, E)$, where $V$ is the set of node operations
(called \emph{ops})
and $E \subseteq V \times V$ are the data flow edges.
Let $n = |V|$ and $m = |E|$.
For simplicity, assume each op~$u$ outputs 
one tensor consumed by (possibly many) downstream ops $v$,
which we represent by the edges $(u,v)$.
This corresponds to a \texttt{tensorflow.Graph}~\citep{abadi2016tensorflow},
and has analogs in MLIR~\citep{lattner2021mlir}, MXNet~\citep{chen2015mxnet}, and PyTorch~\citep{paszke2019pytorch}.

We introduce a few node weights to help model the cost of inference:
\begin{itemize}
    \item $\cost(v)$ is the running time of $v \in V$.
        This is typically estimated with an analytic or learned
        cost model~\citep{kaufman2021learned}.
    \item $\size(v)$ is the memory footprint of the model parameters that $v \in V$ uses.
        For example, if $v$ is a matrix multiplication op,
        $\size(v)$ is the storage cost for the entries of the matrix.
    \item $\outputsize(v) $ is the size of the output of $v$ (\eg in bytes).
\end{itemize}

We use standard graph theory notation to denote dependencies between nodes:
\begin{itemize}
    \item $\innbr(v) = \{u \in V : (u, v) \in E\}$
        is the set of nodes whose output is consumed by $v$.
    \item $\outnbr(v) = \{w \in V : (v, w) \in E\}$
        is the set of nodes that consume the output of $v$.
    \item 
        $\innbr(S) = \bigcup_{v \in S} \innbr(v) \setminus S$ and
        $\outnbr(S) = \bigcup_{v \in S} \outnbr(v) \setminus S$
        extend the neighborhood notation to sets of nodes.
\end{itemize}

The reason for excluding $S$ from the neighborhoods
will become clear when we consider the inter-block communication costs for a partition of $G$.

\subsection{Problem statement}

\paragraph{Acyclic quotient graph constraint}
Let $\cP_{k}(G)$ be the set of partitions of $V$ into $k$ blocks (possibly empty)
such that the induced \emph{quotient graph} of $G$ is acyclic.
Formally, let $P=\{P_1, P_2, \dots, P_k\}$ be a partition of $V$,
i.e.,
$P_1 \cup P_2 \cup \dots \cup P_k = V$ and
$P_i \cap P_j = \emptyset$ for all $i \ne j$.
Since~$P_i$ can be empty,
we can think of partitioning $V$ into at most~$k$ blocks.
For each $v \in V$, let $[v]_P$ denote the block in $P$ containing~$v$.
The quotient graph $Q = (P, E')$ for partition~$P$ has
blocks of $P$ as its nodes
and
reduced edge set
$E' = \{([u]_P, [v]_P) : (u, v) \in E \text{~and~} [u]_P \ne [v]_P\}$.
We require $Q$ to be acyclic so that there is valid data flow
when~$G$ is partitioned across different processors.

\paragraph{Inter-block communication}
Let $B$ be the bandwidth of the interconnect between blocks.
For disjoint sets $S, T \subseteq V$,
the IO cost (in units of time) from $S$ to~$T$ is
\begin{equation}
\label{eqn:iocost_def}
    \iocost(S, T) = \frac{1}{B} \sum_{v \in \innbr(T) \cap S} \outputsize(v).
\end{equation}
We overload singleton notation:
$\iocost(u, v) = \iocost(\{u\},\{v\})$.

By summing over the set of producer ops $v \in \innbr(T) \cap S$,
each tensor going from $S$ to $T$ is counted once,
\emph{even if it has many consumers in $T$.}
This is different from the cost of a traditional edge-cut set since it
considers only one edge in each tensor edge equivalence class (see \Cref{fig:graph_partitions}).
Refining how the cost of communication is modeled is where computation graph
partitioning begins to deviate from more familiar
cut-based graph partitioning problems.

\paragraph{Streaming model parameters}
Each block is a computational unit
with a fixed amount of fast memory,
e.g., SRAM for GPUs
and multi-chip packages~\citep{mei2016dissecting,gao2020estimating,dasari2021apparatus}
and high-bandwidth memory for TPUs~\citep{jouppi2017datacenter,jouppi2023tpu}.
To achieve peak performance,
it is essential that all model parameters assigned to a block be fully cached.
Otherwise, some of these parameters must be streamed to the block during each inference batch
from slow memory, e.g., shared DRAM.
Inter-block bandwidth is typically at least an order of magnitude
slower than intra-block bandwidth~\citep{dao2022flashattention},
so we ignore the communication between ops within a block
and refer to the time needed to stream parameters
that spill over as the \emph{overflow cost} of a block.

There are two key factors for deciding if all model parameters can be fully cached on a block:
(1) the size of its fast memory, and
(2) the peak activation memory.
We start by describing the peak memory scheduling problem~\citep{marchal2019limiting, paliwal2019reinforced, ahn2020ordering,lin2021mcunetv2,vee2021scheduling,zhang2022tensile,fradet2023sequential,jin2023new}.
For a set of ops $P_i$,
the peak memory scheduling problem is to
find a linear execution order of $v \in P_i$
minimizing the amount of working memory
needed for all intermediate computations.
Once we know how much fast memory to reserve for the activations,
we allocate the rest for caching parameters.

The overflow cost for a set of ops $S \subseteq V$
on a block with~$M$ fast memory,
peak memory $\peakmem(S)$, and
inter-block bandwidth $B$ is
\begin{align}
\label{eqn:overflow-cost}
    &\overflowcost(S)
    =
    \frac{\parens*{
        \size(S)
    + \peakmem(S) - \blockmem}^{+}}{B},
\end{align}
where we use the notation $x^+ = \max(x, 0)$.

\paragraph{Total block cost}
The total cost of a block with ops $S \subseteq V$ (\ie its running time)
in a pipeline partition is
\begin{align}
\label{eqn:full-block-cost}
    f(S) &=
        \overbrace{\iocost(V \setminus S, S)}^{\text{input tensors}}
        +
        \sum_{v \in S} \cost(v)
        +
        \overflowcost(S) \notag \\
        &+
        \overbrace{\iocost(S, V \setminus S)}^{\text{output tensors}}.
\end{align}
Putting everything together, we arrive at the following min-max objective function.


\begin{definition}
\label{def:max_throughput_partitioning}
For a computation graph $G$ and number of blocks $k$,
the \emph{max-throughput partitioning problem} (\MTPP) is 
\begin{equation}
\label{eqn:optimization_problem}
    P^* = \argmin_{P \in \cP_k(G)} \set*{ \max_{i \in [k]} f(P_i) },
\end{equation}
where $[k] = \{1,\dots,k\}$.
Let $\OPT = \max_{i \in [k]} f(P^*_i)$ denote the minimum bottleneck cost.
\end{definition}

\begin{remark}
There are \emph{many more} moving parts to ML performance than partitioning model graphs---it is just one of many ML compiler passes
(\eg op fusion, tensor sharding, fine-grained subgraph partitioning, peak memory scheduling).
However, it is one of the highest-order components
with a significant impact on overall system efficiency.
\end{remark}

\section{Mixed-integer programming lower bounds}
\label{sec:lower_bounds}

We first prove that \MTPP is NP-hard and cannot have a fully polynomial-time approximation scheme (FPTAS), unless $\P = \NP$,
by giving a reduction from the minimum makespan scheduling problem on $k$ identical parallel processors,
which is strongly NP-hard~\citep{hochbaum1987using}.
We defer all proofs in this section to \Cref{app:lower_bounds}.

\begin{restatable}{theorem}{TheoremHardness}
\label{thm:hardness}
For $k=2$, \MTPP is \NP-hard.
Furthermore, there does not exist a fully polynomial-time approximation scheme for \MTPP, unless $\P = \NP$.
\end{restatable}

\begin{figure*}
\begin{tcolorbox}[enhanced, opacityframe=1, colback=white!99!black, top=-0.8mm]
\begin{align}
    \text{minimize~~} &\bottleneck \label{eqn:exact_mip} \\
    \text{such that~~} 
        & \bottleneck \ge \blockcost_b  \quad \forall b \in [k] \label{eqn:bottleneck_lower_bounds} \\
        & \blockcost_b = \sum_{v \in V} \cost(v) \cdot x_{vb} + \frac{1}{B} \sum_{u \in V} \outputsize(u) \cdot c_{ub}  \quad \forall b \in [k] \label{eqn:block_cost_def} \\
        & y_{ub} \ge y_{vb} \quad \hspace{2.102cm} \forall (u,v) \in E, b \in [k]
          \hspace{0.5cm} \MyComment{DAG constraints} \label{eqn:dag_constraints} \\
        & c_{u b} \ge y_{u(b-1)} + x_{vb} - 1 \quad \forall(u,v) \in E, b \in [k]
          \hspace{0.5cm} \MyComment{cut input tensors} \label{eqn:input-tensor-cut} \\ 
        & c_{u b} \ge x_{ub} - y_{vb} \quad \hspace{1.185cm} \forall(u,v) \in E, b \in [k]
          \hspace{0.5cm} \MyComment{cut output tensors} \label{eqn:output-tensor-cut} \\
        & c_{u b} \ge 0 \quad \hspace{2.411cm} \forall(u,v) \in E, b \in [k] \notag \\
        & y_{v(b-1)} \le y_{vb} \quad \hspace{1.54cm} \forall v \in V, b \in [k] \notag \\
        & y_{v0} = 0 \quad \hspace{2.39cm} \forall v \in V \hspace{2.2cm} \MyComment{convenience variable} \notag \\
        & y_{vk} = 1 \quad \hspace{2.375cm} \forall v \in V \hspace{2.2cm} \MyComment{boundary condition} \notag \\
        & x_{vb} = y_{vb} - y_{vb-1} \quad \hspace{0.84cm} \forall v \in V, b \in [k] \label{eqn:xy-sub} \\
        & y_{vb}, x_{vb}, c_{vb} \in \{0, 1\} \quad \hspace{0.52cm} \forall v \in V, b \in [k] \label{eqn:binary}
\end{align}
\end{tcolorbox}
\caption{Exact MIP for solving \MTPP, where variables $x_{vb} \in \{0,1\}$ indicate whether node $v \in V$ is assigned to block $b \in [k]$.}
\label{fig:mip}
\end{figure*}

In light of this hardness,
we formulate a novel mixed-integer program to solve \MTPP
and focus on deriving strong lower bounds.
Even for medium-sized models and moderate values of $k$,
the exact program pushes the limits of MIP solvers,
so we relax the formulation to give strong lower bounds
while using fewer variables, constraints, and non-zeros.
The exact MIP in \eqref{eqn:exact_mip} and its relaxations are the main theoretical contribution of
our work, allowing us to provide strong per-instance approximation guarantees.

To simplify the presentation, we ignore the \overflowcost terms in 
\Cref{eqn:full-block-cost} since $\peakmem(S)$ depends on how the ops in a block are scheduled~\citep{paliwal2019reinforced}.
This is equivalent to reserving a buffer for activations in each block
and treating the remaining amount of fast memory as the new budget.

\subsection{Exact MIP formulation} \label{subsec:exact-mip}

We now present the MIP for solving \MTPP in \Cref{fig:mip}.
The main idea is to number the blocks from $1$ to~$k$ in DAG order
(i.e., a topological order of the induced quotient graph)
and use binary decision variables to assign nodes to blocks. 

\paragraph{Decision variables}
There are $O(nk)$ binary variables:
\begin{itemize}
    \item $x_{vb}$ indicates whether node $v \in V$ is assigned to block $b \in [k]$.

    \item $y_{vb}$ indicates whether node $v \in V$ is assigned to some block at or before $b$.
        For any feasible assignment, this means $y_{vk} = 1$, for all $v \in V$,
        and $y_{v(b-1)} \le y_{vb}$, for all $v \in V, b \in [k]$.
        We let $y_{v0} = 0$ for notational convenience.
        
    \item $c_{ub}$ indicates whether any edge $(u,v) \in E$,
    corresponding to the tensor that $u$ produces, flows \emph{into or out of} block $b$.
\end{itemize}

All decision variables are nominally binary in \Cref{eqn:binary}, 
but we can relax the $x$ and $c$ variables to $[0,\infty)$ since they naturally 
lie in $\{0,1\}$ whenever the $y$ variables do.

The $x$ and $y$ variables represent the same information in two ways, and hence are redundant.
Using both, however, allows us to express some constraints more naturally.
In our code, we use \Cref{eqn:xy-sub} 
to eliminate each occurrence of $x_{vb}$.
Doing so offers two advantages relative to eliminating the~$y$ variables. 
First, the tensor cut constraints
require $O(mk^2)$ non-zeros if expressed purely in terms of the $x$ variables. 
Second, and more crucial, branching on the $y$ variables
works \emph{in tandem} with the acyclicity constraints to create
more asymmetry in the branch-and-bound process,
allowing the MIPs to solve faster
compared to branching on the $x$ variables.

\paragraph{Objective value}
The auxiliary $\bottleneck$ variable allows us to minimize the max block cost
via constraint \Cref{eqn:bottleneck_lower_bounds}.
The $\blockcost_b$ variables defined in~\Cref{eqn:block_cost_def}
capture the total node cost assigned to block $b$
plus the induced cut costs, counting each tensor edge in the cut-set exactly once.

\paragraph{DAG constraints}
Given the ``completed-by-block $b$'' variables $y_{vb}$,
we use constraint \Cref{eqn:dag_constraints} to force the quotient graph to be acyclic.
If $(u,v) \in E$ and $v$ is assigned to block~$b$ or earlier, the DAG constraints guarantee that
$u$ is also assigned to block $b$ or earlier, so \Cref{eqn:dag_constraints} holds, and conversely.

\paragraph{Tensor edge-cut constraints}
Each tensor $\tau$ can be represented by multiple edges in $G$,
each with the same source node $u(\tau)$. If \emph{any} of these edges is cut by block $b$,
we must set $c_{ub} = 1$.
Constraint~\Cref{eqn:input-tensor-cut} captures the case where an
edge $(u,v)$ flows into block $b$ from the left---namely that when $v$ is assigned to block 
$b$ (\ie $x_{vb} = 1$) and $u$ is assigned to an earlier block (\ie $y_{u(b-1)} = 1$), 
then $c_{ub}$ is forced to be 1, and otherwise it is not. 
For outgoing edges, if $u$ is assigned to block $b$ (\ie $x_{ub} = 1$) and $v$ is assigned to a 
later block (\ie $y_{vb} = 0$), then constraint~\Cref{eqn:output-tensor-cut} forces $c_{ub} = 1.$

\begin{restatable}{theorem}{TheoremExactMIP}\label{thm:exact-mip}
The mixed-integer program in Eq.~\Cref{eqn:exact_mip} solves the max-throughput partitioning problem
using $O(nk)$ variables, $O(mk)$ constraints, and $O(mk)$ non-zeros.
\end{restatable}

\subsection{Relaxing to a three-superblock formulation}
\label{subsec:relax_to_three_superblocks}

If the exact MIP is too difficult to solve, then
we can use a relaxed ``three-superblock'' formulation whose size does not depend on $k$
to compute a lower bound for \OPT.
The idea is to imagine the bottleneck block,
consolidate all earlier blocks into one superblock, and all later blocks into a third superblock.
Then, we use a combinatorial lower bound $L$ to ensure that the middle block is sufficiently expensive.

\begin{lemma}[Simple lower bound]
\label{lem:simple_lower_bound}
For any computation graph $G=(V,E)$,
cost function $\cost : V \rightarrow \R_{\ge 0}$,
and partition of $V$ into $k \ge 1$ blocks,
there exists a block with at least $L$ units of \cost, where
\begin{equation}
\label{eqn:simple_lower_bound}
    L = \max\parens*{
      \max_{v \in V} \cost(v),
      \frac{1}{k}\sum_{v \in V} \cost(v)
    }
    \le
    \OPT.
\end{equation}
\end{lemma}

This formulation is the same as the exact MIP in \Cref{fig:mip}, for $k=3$,
except for two small changes:
\begin{enumerate}
    \item Add a constraint that forces the node cost of block~$2$ to be at least the simple lower bound in \Cref{lem:simple_lower_bound}:
    \begin{equation}
    \label{eqn:bottleneck_simple_lower_bound}
        \sum_{v \in V} \cost(v) \cdot x_{v2} \ge L.
    \end{equation}

    \item Remove $\blockcost_1$, $\blockcost_3$, and all constraints involving them from the MIP.
    This simplifies the objective to
    \begin{equation*}
    \label{eqn:bottleneck_objective}
        \text{minimize~~} \blockcost_2,
    \end{equation*}
    as the middle block aims to model the bottleneck cost.
\end{enumerate}
Observe that the true bottleneck block can hide inside of superblocks 1 or 3, and hence would not contribute to the objective.
Therefore, this relaxation can give strictly weaker lower bounds than the exact MIP.

\begin{restatable}{corollary}{CorollaryBottleneckLowerBound}
\label{cor:bottleneck_lower_bound}
For any computation graph $G$ and number of blocks $k \ge 1$,
the three-superblock MIP
uses $O(n)$ variables, $O(m)$ constraints, and $O(m)$ non-zeros, and 
gives a lower bound for the $\MTPP$ objective.
\end{restatable}

\subsection{``Guess the bottleneck block'' formulation} \label{subsec:guess-bottleneck-block-id}

The three-superblock MIP is
agnostic about which block in the original instance
(represented by block~2 in the relaxation) is the one with $\cost \ge L$.
Building on this, another approach is to guess that the bottleneck is block $j \in [k]$,
and get a stronger lower bound $\text{LB}_j$ under this assumption.
Since the guess could be wrong, we must compute $\text{LB}_j$ for all $j \in [k]$
and take $\min_{j \in [k]} \text{LB}_j$ as the valid lower bound for $\OPT$.

The formulation for $\text{LB}_j$ is the same as the exact MIP with $k=3$ in \Cref{fig:mip}, except we add constraint~\eqref{eqn:bottleneck_simple_lower_bound}
and make one other change.
    For $b \in \{1, 3\}$, the right-hand side of \cref{eqn:block_cost_def} defining $\blockcost_{b}$
    is the combined node cost and cut cost for superblock~$b$,
    excluding the tensors that are cut by blocks in the same superblock and counting 
    tensors that enter superblock~3 only once, even if their edges terminate in different blocks within superblock~$3$.
    The worst block in the superblock is at least as expensive as the average block,
    so we can replace the constraints in \Cref{eqn:bottleneck_lower_bounds} with the following lower bounds:
    \begin{align*}
        \bottleneck &\ge \frac{1}{j - 1} \cdot \blockcost_1 \\
        \bottleneck &\ge \frac{1}{k - j} \cdot \blockcost_3.
    \end{align*}
If $j=1$, this forces $x_{v1} = 0$ for all $v \in V$;
and if $j = k$, this forces $x_{v3} = 0$ for all $v \in V$.
Said differently, nodes cannot be assigned before block $1$ or after block $k$.
If $k=3$, then there is no reason to prefer one relaxation over the other
(\ie \Cref{subsec:relax_to_three_superblocks} and \Cref{subsec:guess-bottleneck-block-id}),
but for $k \gg 3$,
the two relaxations use substantially fewer variables and constraints than the exact formulation in \Cref{subsec:exact-mip}.

\section{Algorithm}
\label{sec:algorithm}

\begin{table*}[!b]
  \vspace{-0.2cm}
  \caption{Geometric means of the best available lower bound from the MIP hierarchy,
  normalized by the best solution found using BKRGA, across the production dataset.}
  \label{tab:prod_scaled_lower_bounds}
  \centering
  \vspace{0.1cm}
  \begin{tabular}{lcccccccccccccccccccccccc}
    \toprule
    Lower bound & $k=2$ & $k=4$ & $k=8$ & $k=16$ & $k=32$ & $k=64$ \\
    \midrule
       \simple (\Cref{lem:simple_lower_bound}) & 0.8340 & 0.6627 & 0.5236 & 0.4598 & 0.4435 & 0.4401 \\
       \bottleneck & 0.9597 & 0.7911 & 0.6481 & 0.5770 & 0.5590 & 0.5543 \\
       \bottleguess & 0.9901 & 0.8446 & 0.6601 & 0.5780 & 0.5593 & 0.5543 \\
       \exact & 0.9901 & 0.9737 & 0.9588 & 0.9452 & 0.8749 & 0.7874 \\
    \bottomrule
  \end{tabular}
\end{table*}

We now present our approach to pipeline partitioning.
This algorithm is simple by design and runs inside ML compilers with tight latency requirements (e.g., XLA for TensorFlow).
In \Cref{sec:experiments},
we prove that it is near-optimal across a production testbed
by computing per-instance approximation guarantees using our new MIP formulations.

\subsection{Reducing to a search over topological orderings}
We first reduce \MTPP to a search over topological orderings as follows:
\begin{enumerate}
    \item An optimal partition $P^*$ in Eq.~\Cref{eqn:optimization_problem} has a
    corresponding topological order $\pi^*$.

    \item \label{item:node-weights} There exist node weights $\mat{x}^* \in [0, 1]^n$ such that
    Kahn's topological sorting algorithm with tiebreaking on $\mat{x}^*$ recovers $\pi^*$ (see \Cref{app:kahn}).

    \item \label{item:DP} For any topological order $\pi \in \mathfrak{S}_{V}$,
    we can efficiently compute an optimal segmentation of $\pi$
    via dynamic programming. By slicing a topological order this way,
    we easily \emph{satisfy the acyclicity constraint}.
\end{enumerate}

One method for searching over topological orders in \Cref{item:node-weights} is to sample random node weights.
Another is to learn the weights using a
genetic algorithm or reinforcement learning.
To implement the dynamic program in \Cref{item:DP} efficiently, 
we use a fast data structure for segment cost queries.

\begin{restatable}{lemma}{LemmaSegmentCostDataStructure}
\label{lem:segment_cost_data_structure}
There is a $\SegmentCostDataStructure$ that takes
computation graph $G=(V,E)$
and topological order $\pi \in \mathfrak{S}_{V}$ as input,
and supports the following operations:
\begin{itemize}
    \item $\Initialize(G, \pi)$: Preprocesses the graph in $O(n^2 + m \log^2 (n))$ time.
    \item $\Query(\ell, r)$: Returns $f(\{v_{\pi(\ell)}, \dots, v_{\pi(r)}\})$ in Eq.~\cref{eqn:full-block-cost}
    in constant time, after initialization.
\end{itemize}
\end{restatable}

All proofs for this section are deferred to \Cref{app:algorithm}, but
at a high level, \SegmentCostDataStructure computes the cost of each $[\ell,r]$
slice of the topological order $\pi$ (counting each tensor in a cut once)
using a sliding window algorithm and
two-dimensional Fenwick tree~\citep{mishra2013new}.

\begin{algorithm}[t]
\caption{Optimal \MTPP slicing of topological order $\pi$ into at most $k$ blocks.}
\label{alg:dynamic_programming}
\begin{algorithmic}[1]
\Function{\SliceGraph}{$G$, $k$, $\pi$}\\ \MyComment{Partitions the full topological order $\pi$}
    \State Initialize \texttt{segment\_cost} data structure for $(G, \pi$)
    \State \textbf{return} $\Solve(\texttt{segment\_cost}, n, k)$
\EndFunction

\Function{\Solve}{\texttt{segment\_cost}, $\nhat$, $\khat$}\\ \MyComment{Recursively partitions the first $\nhat$ nodes optimally into $\khat$ blocks}
    \If{$\khat=1$}
        \State \textbf{return} $\texttt{segment\_cost}.\texttt{Query}(1, \nhat)$
    \EndIf
    \vspace{-0.15cm}
    \State $ans \gets \infty$
    \For{$\ell=1$ to $\nhat$}
        \State $a \gets \Solve(\texttt{segment\_cost}, \ell, \khat-1)$
        \State $b \gets \texttt{segment\_cost}.\texttt{Query}(\ell + 1, \nhat)$
        \State $ans \gets \min(ans, \max(a, b))$ \label{line:recurrence}
    \EndFor
    \State \textbf{return} $ans$
\EndFunction
\end{algorithmic}
\end{algorithm}

\begin{restatable}{lemma}{LemmaSliceGraph}
\label{lem:slice_graph}
\SliceGraph runs in time $O(n^2 k + m\log^2 n)$ and finds an optimal
segmentation of topological order $\pi$
into at most $k$ blocks for \MTPP.
\end{restatable}

\subsection{Searching over topological orders}
\label{sec:searching}

Any topological order $\pi \in \mathfrak{S}_V$ can be realized using 
Kahn's algorithm with the right node priorities $\mat{x} \in [0, 1]^n$.
Thus, we now focus on methods for finding good vectors of node weights.
Note that we only recast \MTPP into a search
over topological orders to bypass the acyclic quotient graph constraint---it is still \emph{provably hard} to find an optimal topological order.
Our next result formalizes this observation.

\begin{restatable}{theorem}{Main}
\label{thm:main}
There exist node priorities $\mat{x}^* \in [0, 1]^n$ such that
Kahn's algorithm with tie-breaking outputs a topological order
$\pi^*$ for which $\SliceGraph(G, k, \pi^*) = \OPT$.
\end{restatable}

Since it is NP-hard to find an optimal topological order $\pi^*$,
we explore heuristics that are fast and work well in practice.

\emph{Random weights.}
    Each weight $x_i \sim U(0, 1)$ is drawn i.i.d.\ from the uniform distribution.
    Note that this is not the same as sampling topological orders uniformly at random,
    which has its own rich history~\citep{matthews1991generating,bubley1999faster,huber2006fast,garcia2019bottom}.

\emph{Biased random-key genetic algorithm (BRKGA).}
    BRKGA is a problem-agnostic metaheuristic that evolves real-valued vectors $\mat{x} \in [0, 1]^N$ (called \emph{chromosomes})
    using a \emph{decoder} function
    that links BRKGA's evolutionary rules
    to the problem at hand~\citep{gonccalves2011biased}.
    We use BRKGA to optimize the node priorities $\mat{x}$ by evaluating the quality of its induced optimal segmentation $P(\mat{x})$.
    In the language of genetic algorithms,
    the \emph{fitness} of $\mat{x}$ is the $\MTPP$ objective
    $\max_{i \in [k]} f_{P(\mat{x})}(i)$.
    We present this decoder in \Cref{alg:decoder}.

\begin{algorithm}
\caption{BRKGA decoder using Kahn's algorithm and \SliceGraph to partition $G$.}
\label{alg:decoder}
\begin{algorithmic}[1]
    \Function{\texttt{BrkgaSortAndSliceDecoder}}{$G$, $k$, chromosome $\mat{x} \in [0, 1]^n$ of node priorities}
    \State $\pi \gets \Kahn(G, \mat{x})$
    \State \textbf{return} $\SliceGraph(G, k, \pi)$
    \EndFunction
\end{algorithmic}
\end{algorithm}

\section{Experiments}
\label{sec:experiments}

\begin{figure*}
\centering
\includegraphics[width=1.0\textwidth]{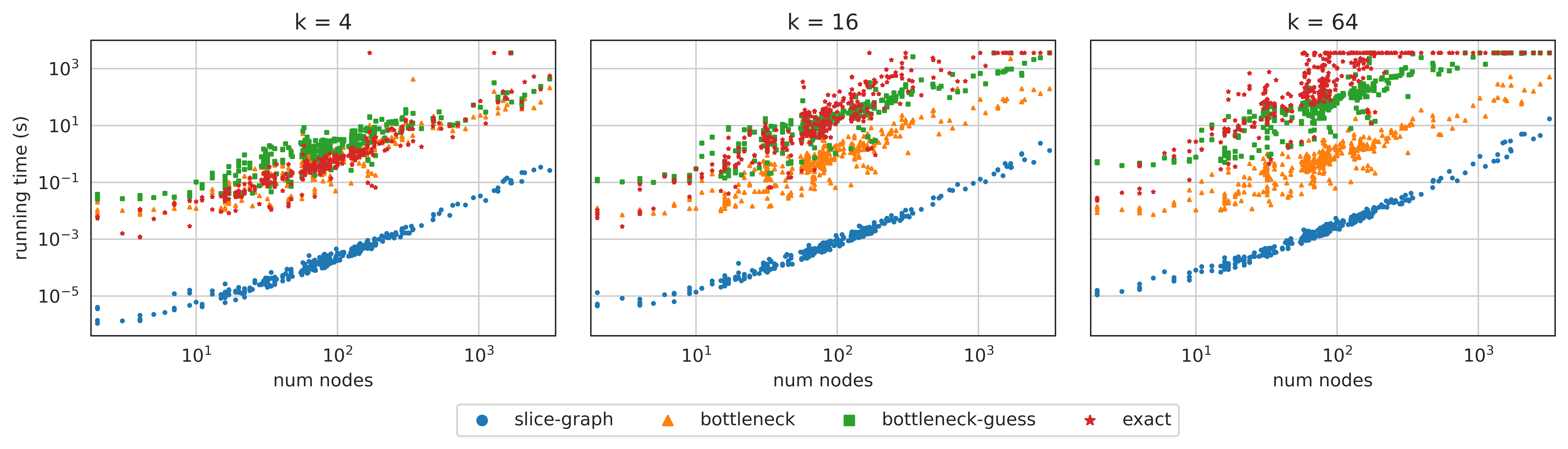}
\vspace{-0.75cm}
\caption{Running times of $\SliceGraph$ and different MIP lower bound computations across the production models. Each point denotes a run for one graph, color-coded to denote $\SliceGraph$ partitioning vs.\ \bottleneck, \bottleguess, and \exact lower bounds. The \bottleguess times are summed across all $k$ MIP instances involved. Each plot is for a different value of $k$.
In order to facilitate visual comparisons across the plots, all three employ the same $y$-axis.
Some of the data tops out at 3600 seconds
since that is where we set the MIP time limit.
}
\vspace{-0.3cm}
\label{fig:production_running_times}
\end{figure*}

We now present an empirical study of our MIP lower bound formulations.
To motivate this, we revisit the practitioner's dilemma (\Cref{sec:dilemma})
with a case study
for a specific model in our production dataset called
\texttt{net10\_new\_ckpt-4000009} for $k=4$ blocks.
Suppose you want to optimize this model to run 
efficiently in production on millions of dollars worth of hardware.
You use \texttt{BrkgaSortAndSliceDecoder} to partition the model,
and the best solution you find has a bottleneck time of 3454.79.
Is this good or bad? How hard should you work to improve it?

Since the units are arbitrary, we rescale the solution to have value 1. Applying 
\Cref{lem:simple_lower_bound}, you prove the \simple combinatorial lower bound of 0.2775 
for this instance (in the rescaled units). This is not enough to satisfy your boss
since you could hypothetically reduce your costs by 72\%. Next you turn to the 
\bottleneck lower bound in \Cref{subsec:relax_to_three_superblocks}, which computes to 0.6570 
(\ie 2.37x stronger than \simple). This makes your boss less cranky, but there is still 
an uncomfortable gap, so you compute \bottleguess 
in \Cref{subsec:guess-bottleneck-block-id}, yielding another 1.39x increase in your 
lower bound to 0.9152. This is enough to satisfy your boss, but for your own curiosity 
you run Gurobi on the \exact MIP, and arrive at a lower bound of 0.9913. As it turns out, 
your original solution was nearly optimal the whole time,
and you just didn't know it!
However, this increasing sequence of certificates---obtained with increasing computational effort---steadily increased your confidence.

Our experiments repeat this exercise for hundreds of computation graphs across our production testbed, for all $k \in \{2, 4, 8, 16, 32, 64\}$.
We rescale each lower bound to represent it as a fraction of the best solution value found,
and then we summarize the quality of the lower bounds in
\Cref{tab:prod_scaled_lower_bounds}
by taking their geometric mean across all graphs.

\vspace{-0.15cm}
\paragraph{Datasets}
Our production testbed is a superset of the models
in the experiments of \citet{xie2022transferable}.
This includes 369 computation graphs from many application domains.
(\eg BERT, ResNet, MobileNet, vision models, LSTMs, speech encoders, and WaveRNN).
Most of these graphs are publicly available,
but we cannot publish the node and edge weights as they come from an internal proprietary cost model.

We also run the same experiment on 1000 publicly available
synthetic model graphs from REGAL~\citep{paliwal2019reinforced}.
See \Cref{sec:regal} for the results and more details.

\vspace{-0.15cm}
\paragraph{Setup}
We solve the MIPs using a combination of Gurobi v9.0.2 \cite{gurobi} and SCIP v7.0.1 \cite{BestuzhevaEtal2021OO}.
Each instance is run on a heterogeneous cluster
containing, \eg Intel Xeon Platinum 8173M @ 2.00GHz processors,
and the best lower bound proven in a fixed time limit is reported.

\vspace{-0.15cm}
\paragraph{Partitioning algorithms}
We compare several topological sort heuristics for \SliceGraph:
random weights, BRKGA, and minimum linear arrangement.
Overall, BRKGA worked best,
so we use its solutions to normalize the lower bounds in \Cref{tab:prod_scaled_lower_bounds}.
We provide more details in \Cref{app:primal-heuristics}.

\vspace{-0.15cm}
\paragraph{Results}
We discuss several interesting properties
about the results in \Cref{tab:prod_scaled_lower_bounds}.
First,  
\bottleguess and \exact produce the same lower bounds for $k=2$.
This is not a coincidence since the MIPs are equivalent for $k=2$. 
For larger values of $k$, \bottleguess MIPs can ``cheat'' in two
ways:
(1) by ignoring communication costs within the superblocks,
and (2) effectively counting the total work of each superblock as if it is smeared uniformly across its blocks. 
As the value of $k$ increases, so does the opportunity to cheat
since a larger fraction of the communication cost is ignored and the smearing effect is more pronounced.
Hence, \bottleguess loses ground relative to \exact.
The advantage of \bottleguess over 
\bottleneck is that it considers edges crossing the superblock boundary. However, larger 
values of $k$ dilute this advantage since the communication costs get amortized over 
a larger number of blocks implicitly contained in each superblock. 
By $k=64$, the lower bounds produced by \bottleguess and \bottleneck are the same.

Interestingly, the additive gaps between \bottleneck and \simple in \Cref{tab:prod_scaled_lower_bounds} hold for each $k$, hovering near 0.12.
This makes sense because the advantage of \bottleneck is that it considers communication costs while \simple does not,
and its treatment of these costs does not depend on $k$.

The running times show that \bottleguess and \exact are about four orders of magnitude slower than \SliceGraph
in \Cref{fig:production_running_times}.
Therefore, this MIP-based 
analysis is most valuable for offline analysis, rather than running as part of the compiler. 
For $k \leq 16$, \bottleguess running times are roughly on par with \exact. 
We note that \bottleguess in \Cref{fig:production_running_times}
depicts the sum of the running times for all $k$ sub-MIP solves when guessing the $k$ bottlenecks to compute the lower bound.
Since these solves are independent,
they could be run in parallel to achieve a faster wall-clock time.
Further, MIP solve times tend to be strongly superlinear in the problem size (observe the log-log scale),
so we would expect \bottleguess to show an advantage even with total running time for larger $k$.
Indeed, this behavior clearly emerges by $k=64$.

\section{Related work}
\label{sec:related_work}

\paragraph{Model parallelism}
Two of the seminal works on pipeline parallelism for machine learning are
GPipe~\citep{huang2019gpipe} and
PipeDream~\citep{narayanan2019pipedream, narayanan2021memory}.
These works focus on scaling up \emph{DNN training}
and have spawned a long list of related work:
\texttt{torchgpipe}~\citep{kim2020torchgpipe},
FPDeep~\citep{wang2020fpdeep},
Pipemare~\citep{yang2021pipemare},
TeraPipe~\citep{li2021terapipe},
BaPipe~\citep{zhao2022bapipe},
SAPipe~\citep{chen2022sapipe},
BPipe~\citep{kim2023bpipe},
synchronous pipeline planning~\citep{luo2022efficient},
scheduling in heterogenous settings~\citep{park2020hetpipe,yuan2022decentralized},
breadth-first pipeline parallelism~\citep{lamy2023breadth},
and SWARM parallelism~\citep{ryabinin2023swarm}.

Of these works,
PipeDream considers the most similar mathematical model \citep[Section 3.1]{narayanan2019pipedream}.
There is some overlap with \MTPP (\eg a min-max running time objective),
but there are also major differences:
\begin{enumerate}
    \item PipeDream assumes a single topological order
    and that the induced quotient graph is a \emph{path graph}.
    This means IO can only flow between adjacent blocks instead of to shared memory for downstream consumers.
    \item PipeDream uses communication-work concurrency,
        so the running time of a processor executing ops $S \subseteq V$ is
        $\max(\cost(S), \iocost(V\setminus S, S) + \iocost(S, V\setminus S))$.
    \item PipeDream supports data parallelism and is designed to use replicated workers for training.
        The model weight updates for replicated workers use \emph{weight stashing} and a syncing technique
        that is not necessary for inference.
\end{enumerate}

Another technique for going beyond data parallelism is \emph{tensor sharding},
\eg
Mesh TensorFlow~\citep{shazeer2018mesh},
Megatron-LM~\citep{shoeybi2019megatron},
GShard~\citep{lepikhin2021gshard},
and GSPMD~\citep{xu2021gspmd}.

\paragraph{Acyclic graph partitioning}
The origins of acyclic graph partitioning are in multiprocessor scheduling~\citep{garey1979computers}.
There have since been several key applications for pipeline parallelism~\citep{cong1994acyclic, gordon2006exploiting,sanchez2011dynamic}.
The computational hardness and inapproximability of balanced acyclic partitioning has recently been revisited in~\citet{moreira2017graph,papp2023partitioning}.
Some practical methods for acyclic graph partitioning use graph coarsening~\citep{moreira2018evolutionary,herrmann2019multilevel,popp2021multilevel}
and MIP with branch-and-bound solvers~\citep{nossack2014branch,albareda2019reformulated,ozkaya2022simple}.

\section*{Conclusion}
\label{sec:conclusion}

This work formalizes $\MTPP$ for pipelined DNN inference,
proposes novel mixed-integer programs for computing lower bounds for this partitioning objective,
and presents fast and effective partitioning algorithms for maximizing inference throughput.
Our lower bounds allow us to prove strong a posteriori approximation guarantees,
which can be invaluable in practice since countless software engineering hours
are spent partitioning mission-critical ML models across accelerators to improve system efficiency.
Without good lower bounds, it is often unclear
if practitioners should continue searching for better partitions,
or if they are near optimality and just don't know it.
Our MIP formulations allow us to compute certificates that
act as a stopping condition on further investment of software engineering time.

\section*{Impact statement}
We present work that advances the field of ML efficiency.
There are many potential societal consequences of our work, none which we feel must be specifically highlighted here.

\section*{Acknowledgements}
We thank Dong Hyuk Woo for encouraging us to research pipeline partitioning algorithms.
Part of this work was done while Kuikui Liu was an intern at Google Research.

\bibliography{references}
\bibliographystyle{icml2024}

\newpage
\appendix
\onecolumn

\section{Missing analysis for \Cref{sec:lower_bounds}}
\label{app:lower_bounds}

\TheoremHardness*

\begin{proof}
The minimum makespan scheduling problem on $k$ identical parallel processors is as follows.
We are given processing times for $n$ jobs $(p_1, p_2, \dots, p_n)$
and asked to find an assignment of jobs to processors
so that the completion time (i.e., \emph{makespan}) is minimized.
We reduce to \MTPP by constructing a graph $G$ with $n$ vertices and no edges,
setting $\cost(v_i) = p_i$,
and computing a max-throughput partition of $G$ to solve the original makespan instance.

For $k=2$, this is the $\NP$-hard \emph{partition problem}.
More generally, minimum makespan scheduling is strongly \NP-hard,
so there cannot exist an FPTAS, unless $\P = \NP$~\citep{hochbaum1987using}.
\end{proof}

\TheoremExactMIP*

\begin{proof}
The $x$, $y$, and $c$ variables are each indexed over all nodes and blocks, so there are $O(nk)$ 
variables total. Constraints \crefrange{eqn:dag_constraints}{eqn:output-tensor-cut} are each indexed 
over all edges and blocks, so there are $O(mk)$ of those. Each constraint except for
\cref{eqn:block_cost_def} has a constant number of non-zeros, so they contribute $O(mk)$ 
non-zeros. Each of the $k$ constraints of type \cref{eqn:block_cost_def} has $n$ each of the~$x$ 
and $c$ variables, so $nk$ non-zeros overall. Thus, there are $O(mk)$ non-zeros in total.

To prove this MIP correctly models \MTPP, we must prove 
that (a) every solution to the problem corresponds to a solution of the MIP (with the same 
objective value), and (b) every solution to the MIP can be transformed into a solution with the same 
or better cost that corresponds to a solution of the problem (with the same objective value).

To prove (a), start with any $\MTPP$ solution. Each node $v$ is assigned to 
exactly one block $b$, so set $x_{vb} = 1$ and $x_{v \bhat} = 0$, for all $\bhat \neq b$. Set $y$ to 
match, i.e., $y_{v0} = \dots = y_{v(b-1)} = 0$ and $y_{vb} = \dots = y_{vk} = 1$.
For each edge 
$(u,v)$ and block~$b$, set $c_{ub} = 1$ if the edge crosses into or out of the block, and 0 otherwise. 
Finally, set $\blockcost_b$ to satisfy \Cref{eqn:block_cost_def} and set $\bottleneck$ to be the 
maximum of the block costs.

We must now check that all of the constraints are satisfied.
Constraints \Cref{eqn:bottleneck_lower_bounds} 
and \Cref{eqn:block_cost_def} are satisfied by construction. Since the original solution 
satisfies the 
DAG constraints, each edge goes from some block to the same or a later block, so constraint~\Cref{eqn:dag_constraints} is satisfied. Constraint \Cref{eqn:input-tensor-cut} cannot be violated 
unless $y_{u(b-1)} = x_{ub} = 1,$ because otherwise the RHS is zero or negative. But in this case, edge 
$(u,v)$ originates before block $b$ and terminates inside block $b$, so the edge is cut as an input 
tensor, so we set $c_{ub} = 1,$ satisfying constraint \Cref{eqn:input-tensor-cut}. Similarly, the only 
way constraint \Cref{eqn:output-tensor-cut} can be violated is if $x_{ub} = 1$ and $y_{vb} = 
0$. In this case, edge $(u,v)$ originates in block~$b$ and terminates after block $b$, so the edge is 
cut as an output tensor, so we have set $c_{ub} = 1,$ satisfying constraint 
\Cref{eqn:output-tensor-cut}. Constraint \Cref{eqn:xy-sub} and the~$y$ monotone ordering constraints are 
satisfied by construction. Finally, $\bottleneck$ really does capture the objective value since it 
equals the cost of the most expensive block, and the block costs are defined in \Cref{eqn:block_cost_def}. 
Therefore, every solution to $\MTPP$ corresponds to a solution of the MIP with the same cost.

Now we prove property (b). 
First, note that constraints \Cref{eqn:input-tensor-cut} and \Cref{eqn:output-tensor-cut} each 
place one lower bound on $c_{ub}$ for each edge $(u,v) \in E$.
Since the only other place $c_{ub}$ appears is in the objective function (implicitly via constraints \Cref{eqn:block_cost_def} and 
\Cref{eqn:bottleneck_lower_bounds}), setting $c_{ub}$ to the maximum of those lower bounds can only 
improve the objective function without harming feasibility. Similarly, $\bottleneck$ should be set to 
the maximum of the lower bounds in \Cref{eqn:bottleneck_lower_bounds}. For a fixed $v \in V$, 
the $y_{vb}$ variables start at 0 when $b=0$ and end at 1 when $b=k$, and by \Cref{eqn:xy-sub} 
we have $x_{vb} = 1$ for the 
value of $b$ when $y_{vb}$ first jumps up to~1.
Thus, the set $\{v \in V : x_{vb} = 1\}$ 
defines the $b$-th block of the partition, and these blocks disjointly cover all 
nodes $v \in V$. Moreover, by a similar argument as above, if any of the edges $(u,v) \in E$ forces 
$c_{ub} = 1$ via constraints \Cref{eqn:input-tensor-cut} or \Cref{eqn:output-tensor-cut} then edge 
$(u,v)$ really is cut by block $b$ in this partition, and otherwise none of the edges out of $u$ is 
cut by block $b$.
Thus, \Cref{eqn:block_cost_def} captures the cost of each 
block in this solution, and $\bottleneck$ captures the cost of the bottleneck block.
\end{proof}

\CorollaryBottleneckLowerBound*

\begin{proof}
The three-superblock MIP is essentially the same as the formulation in 
\Cref{fig:mip} for $k=3$, which means $k$ gets absorbed in the big-$O$ notation and 
the sizes become $O(n)$ variables, $O(m)$ constraints, and $O(m)$ non-zeros.

To prove this MIP gives a valid lower bound,
we start with any solution $P^*$ to $\MTPP$ and generate a MIP solution whose value is the same or lower.
\Cref{lem:simple_lower_bound} shows that some block must be assigned at least $L$ units of work, so find one such block $b$ in $P^*$.
For each node $v$, set $x_{v2} = 1$ if $v$ is in block $b$, $x_{v1}=1$ if $v$ is in an earlier block, and $x_{v3}=1$ if $v$ is in a later block.
Set all other $x$ variables to 0, the $y$ variables as implied by constraints \Cref{eqn:xy-sub},
and the $c$ variables to the minimum value that satisfies constraints \Cref{eqn:input-tensor-cut} and \Cref{eqn:output-tensor-cut}.
We have satisfied the work constraint by construction, and by the same reasoning as in the proof of \Cref{thm:exact-mip}, the value of the MIP solution we constructed equals the cost of block~$b$, which is at most the cost of the bottleneck block in the partition.
Since this construction works for all solutions to $\MTPP$,
we have proven that the optimal solution of the three-superblock MIP is a lower bound for the true optimal solution.
\end{proof}

\section{Missing algorithms and analysis for \Cref{sec:algorithm}}
\label{app:algorithm}

\subsection{Kahn's algorithm with node priorities}
\label{app:kahn}

Kahn's algorithm is a topological sort algorithm that repeatedly peels off the leaves of a DAG~\citep{kahn1962topological}.
It is particularly useful because it can output different orderings---if there are multiple leaves, different tie-breaking rules produce different topological orders.
We give pseudocode for \Kahn in \Cref{alg:kahn},
which takes a vector $\mat{x} \in [0, 1]^n$ of node priorities as input,
and runs in time $O(n\log n + m)$ if implemented with a heap-based priority queue for the active set of leaves.

\begin{algorithm}[H]
\caption{Kahn's topological sorting algorithm with tie-breaking by node priorities.}
\label{alg:kahn}
\begin{algorithmic}[1]
\Function{\texttt{KahnWithNodePriorities}}{graph $G = (V, E)$, node priorities $\mat{x} \in [0, 1]^n$}
\State Initialize $\pi \gets \mat{0}_{n}$, $\texttt{indegree} \gets \mat{0}_n$, and $i \gets 1$
\For{each $v \in V$}
    \State $\texttt{indegree}[v] \gets |N^{-}(v)|$
\EndFor
\State Initialize priority queue $q$ \phantom{+++} \MyComment{Max heap implementation}
\For{each leaf node $v \in V$}
    \State Insert priority-node pair $(x_v, v)$ into $q$
\EndFor
\While{$q$ is not empty}
    \State $u \gets \text{top}(q)$; $\text{pop}(q)$
    \State $\pi[i] \gets u$
    \For{each $v \in N^+(u)$}
        \State $\texttt{indegree}[v] \gets \texttt{indegree}[v] - 1$
        \If{$\texttt{indegree}[v] = 0$}
            \State Insert priority-node pair $(x_v, v)$ into $q$
        \EndIf
    \EndFor
    \State $i \gets i + 1$
\EndWhile
\State \textbf{return} $\pi$
\EndFunction
\end{algorithmic}
\end{algorithm}

\subsection{Segment cost data structure}
\label{sec:segment_cost_data_structure}

We start with a simpler version of the segment cost data structure that makes entrywise updates to the
$\texttt{io\_struct[][]}$ array during initialization.
We give a proof of its correctness,
and then we explain how to speed up this preprocessing step with a two-dimensional Fenwick tree
to achieve faster $O(\log^2(n))$ subrectangle updates~\citep{mishra2013new}.

\begin{warmup}
\label{warmup:warmup}
There is a $\SegmentCostDataStructure$ that takes a
computation graph $G=(V,E)$
and topological order $\pi \in \mathfrak{S}_{V}$ as input,
and supports the following operations:
\begin{itemize}
    \item $\Initialize(G, \pi)$: Preprocesses the graph in $O(n^3)$ time.
    \item $\Query(\ell, r)$: Returns $f(\{v_{\pi(\ell)}, \dots, v_{\pi(r)}\})$ in Eq.~\cref{eqn:full-block-cost}
    in constant time, after initialization.
\end{itemize}
\end{warmup}

\begin{proof}
The correctness is clear by inspection since we are memoizing the contributions of $\cost$, $\iocost$, and $\size$
to the overall cost of each segment.
Calls to $\Query(\ell,r)$ take constant time since we are only performing $O(1)$ array look-ups and arithmetic operations.
It remains to show that the memoization data structures can be built in $O(mn^{2})$ time.

The prefix-sum data structures \texttt{work\_struct} and \texttt{mem\_struct} can all be built in $O(n)$ time since we take one pass over each vertex of the computation graph, and updating each entry requires $O(1)$ time (assuming $O(1)$-time queries to $\cost$ and $\size$).
Furthermore, \texttt{io\_struct} can be constructed in $O(n^{3})$ time.
To see this, observe that we take a single pass over all vertices $u \in V$,
and for each vertex, we perform at most $O(n^2)$ arithmetic operations
and calls to $\outputsize$ since each update is for a distinct $[\ell, r]$ interval.
\end{proof}

\begin{algorithm}[H]
\caption{Segment cost data structure for blocks of the form $P = \{v_{\pi(\ell)}, v_{\pi(\ell + 1)}, \dots, v_{\pi(r)}\}$.}
\label{alg:segment_cost_data_structure}
\begin{algorithmic}[1]
\Function{\Initialize}{$G$, $\pi$}
    \State $\texttt{work\_struct} \gets \texttt{InitWorkStruct}(G, \pi)$
    \State $\texttt{mem\_struct} \gets \texttt{InitMemStruct}(G, \pi)$
    \State $\texttt{io\_struct} \gets \texttt{InitIOStruct}(G, \pi)$
\EndFunction

\Function{\Query}{$\ell$, $r$}
    \State Query $\cost \gets \texttt{work\_struct}[r] - \texttt{work\_struct}[\ell - 1]$
    \State Query $\size \gets \texttt{mem\_struct}[r] - \texttt{mem\_struct}[\ell - 1]$
    \State Set $\texttt{overflow} \gets \size + \peakmem([v_{\pi(\ell)}, v_{\pi(\ell + 1)}, \dots, v_{\pi(r)}]) - \blockmem$
    \State Update $\texttt{overflow} \gets \frac{1}{B}\max\left\{\texttt{overflow}, 0\right\}$
    \State Query $\iocost \gets \texttt{io\_struct}[\ell][r]$
    \State \textbf{return} $\cost + \texttt{overflow} + \iocost$
\EndFunction
\end{algorithmic}
\end{algorithm}

In \Cref{alg:segment_cost_helpers},
we use the fact that $\pi \in \mathfrak{S}_{V}$ is a permutation of the vertices
and that $\pi^{-1} : V \rightarrow [n]$ tells us the index at which
a given node appears in the topological order.

Now we demonstrate how the preprocessing time can be reduced using
subrectangle range updates to subtract $\outputsize(u)$ from disjoint
regions of the two-dimensional $\texttt{io\_struct[][]}$ array.

\begin{algorithm}[H]
\caption{Segment cost data structure helper functions.}
\label{alg:segment_cost_helpers}
\begin{algorithmic}[1]
\Function{\texttt{InitWorkStruct}}{$G$, $\pi$}
\\ \MyComment{Builds 1D array of $\cost$ prefix sums}
    \State Initialize $\texttt{work\_struct} \gets \mat{0}_{n}$
    \State $\texttt{work\_struct}[1] \gets \cost(\pi(1))$
    \For{$i=2$ to $n$}
        \State $\texttt{work\_struct}[i] \gets \texttt{work\_struct}[i-1] + \cost(\pi(i))$
    \EndFor
    \State \textbf{return} \texttt{work\_struct}
\EndFunction

\Function{\texttt{InitMemStruct}}{$G$, $\pi$}
\\ \MyComment{Builds 1D array of $\size$ prefix sums}
    \State Initialize $\texttt{mem\_struct} \gets \mat{0}_{n}$
    \State $\texttt{mem\_struct}[1] \gets \size(\pi(1))$
    \For{$i=2$ to $n$}
        \State $\texttt{mem\_struct}[i] \gets \texttt{mem\_struct}[i-1] + \size(\pi(i))$
    \EndFor
    \State \textbf{return} \texttt{mem\_struct}
\EndFunction

\Function{\texttt{InitIOStruct}}{$G$, $\pi$}
\\ \MyComment{Builds 2D array of $\iocost$ segment costs}
    \State Compute $\texttt{total} \gets \sum_{v \in V} \outputsize(v)$
    \State Initialize $n \times n$ array \texttt{io\_struct} with \texttt{total}
    \For{$u \in V$} \phantom{+++} \MyComment{Remove $\outputsize(u)$ from eligible segments}
        \State Let $S = (v_1, v_2, \dots, v_{d})$ be the nodes in $\{u\} \cup N^+(u)$
        sorted s.t.\ $\pi^{-1}(v_{i}) < \pi^{-1}(v_{i+1})$
        \For{$1 \le \ell \le r < \pi^{-1}(v_1)$} 
            \State $\texttt{io\_struct}[\ell][r] \gets \texttt{io\_struct}[\ell][r] - \outputsize(u)$
        \EndFor
        \For{$i=1$ to $d-1$}
            \For{$\pi^{-1}(v_{i}) < \ell \le r < \pi^{-1}(v_{i+1})$}
                \State $\texttt{io\_struct}[\ell][r] \gets \texttt{io\_struct}[\ell][r] - \outputsize(u)$
            \EndFor
        \EndFor
        \For{$\pi^{-1}(v_{d}) < \ell \le r \le n$}
            \State $\texttt{io\_struct}[\ell][r] \gets \texttt{io\_struct}[\ell][r] - \outputsize(u)$
        \EndFor
        \For{$1 \le \ell \le \pi^{-1}(v_1)$} \phantom{+++} \MyComment{Segments fully containing $\{v_1, v_d\}$}
            \For{$\pi^{-1}(v_{d}) \le r \le n$}
                \State $\texttt{io\_struct}[\ell][r] \gets \texttt{io\_struct}[\ell][r] - \outputsize(u)$
            \EndFor
        \EndFor
    \EndFor
    \State \textbf{return} \texttt{io\_struct}
\EndFunction
\end{algorithmic}
\end{algorithm}

\LemmaSegmentCostDataStructure*

\begin{proof}
We use a two-dimensional Fenwick tree~\citep{mishra2013new} to implement the
$\texttt{io\_struct}$ array.
This data structure needs $O(n^2)$ time and space to initialize as $\mat{0}_{n \times n}$.
It also supports the operation
$\texttt{Update}(p, q, x)$ where $p = (i_1, j_1)$ and $q = (i_2, j_2)$
define two corners of a rectangular sub-array, and adds $x$ to all entries
$\texttt{io\_struct}[i][j]$ for all $(i, j) \in [i_1, i_2] \times [j_1, j_2]$
in time $O(\log^2(n))$.
Therefore, we can first initialize all entries to $\texttt{total}$ in $O(\log^2(n))$ time.

For each $u \in V$, we describe how to update $\texttt{io\_struct}$ in
$O(\deg^+(u) \log^2(n))$ time.
First, observe that there are $\deg^+(u) + 2$ updates in Lines~17--24 of the form
$i \le \ell \le r \le j$. It follows that we can call
$\texttt{Update}((\ell, \ell), (r,r), -\outputsize(u))$
to correctly update the Fenwick tree.
Note that this updates entries that will never be queried (i.e., when $\ell > r$),
but this is not a problem.
Finally, we call
$\texttt{Update}((1, \pi^{-1}(u)), (\pi^{-1}(v_d),n), -\outputsize(u))$ to 
update segments that fully contain $\{u, v_d\}$.
Putting everything together, the total running time to maintain
$\texttt{io\_struct}$ as a two-dimensional Fenwick tree is
\[
    O\parens*{n^2} + \sum_{u \in V} O\parens*{\deg^{+}(u) \log^{2}(n)}
    =
    O\parens*{n^2 + m \log^{2}(n)}.
\]
Since each element-wise query takes $O(\log^2(n))$ time in isolation, we use the
fact that we can iterate over all $O(n^2)$ entries of the Fenwick tree
and write the values of $\texttt{io\_struct}[\ell][r]$ in a separate two-dimensional array
in amortized $O(n^2)$ time.
This allows us to achieve $O(1)$ time queries for the segment cost data structure after initialization.
Correctness follows from \citet{mishra2013new} and \Cref{warmup:warmup}.
\end{proof}

\subsection{Analysis of the \SliceGraph algorithm}

\LemmaSliceGraph*

\begin{proof}
For any topological order $\pi \in \mathfrak{S}_V$, 
initialize $\texttt{segment\_cost}$
for $G(\pi)$ in $O(n^2 + m \log^2(n))$ time and $O(n^2)$ space by \Cref{lem:segment_cost_data_structure}.
Then, run the dynamic programming algorithm
$\Solve(\texttt{segment\_cost}, n, k)$ on the full topological order $\pi$,
which can be done in $O(n^2 k)$ time and $O(n^2)$ space
since there are $O(nk)$ states and each
state can be computed recursively in $O(n)$ time after preprocessing all $[\ell, r]$ segment costs.
This gives an optimal stars-and-bars
partition of $\pi$ into $k$ (possibly empty) blocks for the \MTPP objective in
Eq.~\Cref{eqn:optimization_problem}, which proves the result.
\end{proof}

\Main*

\begin{proof}
Let $P^*$ be an optimal \MTPP partition of the nodes,
and let $Q = (P^*, E')$ be the induced acyclic quotient graph on the blocks.
Let $\sigma$ be a topological ordering of the blocks $P^*$ in $Q$.
Then, set $x^*_{v} \gets k - \sigma^{-1}[[v]_{P^*}]$ for every $v \in V$,
where $[v]_{P}$ denotes the block index $i \in [k]$ of the partition $P=\{P_1, P_2, \dots, P_k\}$
and $\sigma^{-1}$ is the inverse permutation of $\sigma \in \mathfrak{S}_{P^*}$.
This means nodes appearing in the first block of $P^*$ according to $\sigma$ have highest priority.
Running Kahn's algorithm with $\mat{x}^*$ recovers
a topological order $\pi^* \in \mathfrak{S}_{V}$ such that when optimally segmented into $k$ blocks,
has an objective value that is equal to partition $P^*$.
\end{proof}

Even with optimal topological order slicing,
there exist worst-case instances $(k, G, \pi)$, for any $k \ge 2$,
that can be as bad as the trivial $\MTPP$ algorithm
that puts all nodes into the same block.

\begin{restatable}{lemma}{HardInstance}
\label{lem:hard_instance}
For any $k \ge 2$,
there is a computation graph $G$ and topological order $\pi$
such that when optimally sliced,
$\SliceGraph(G, k, \pi) = k \cdot \OPT$.
\end{restatable}

\begin{proof}
Consider a graph with $n = 2k$ nodes of two types:
nodes $\{1,2,\dots,k\}$ are \emph{heavy} with weight $\cost(v_i) = 1-\varepsilon$, and
nodes $\{k+1,k+2,\dots,2k\}$ are \emph{light} with weight $\cost(v_i) = \varepsilon$.
Add one directed edge $(1, k+1)$ with weight $10k$.
Clearly $\OPT = 1$ since~$G$ can be partitioned as
$P^* = \{\{1, k+1\}, \{2, k+2\}, \dots, \{k, 2k\}\}$,
which means $f(P_i^*) = (1 - \varepsilon) + \varepsilon = 1$ for all $i \in [k]$
since each block has no incoming or outgoing edges.

Now consider the topological order $\pi=(1,2,\dots,k,2k,2k-1,\dots,k+1)$.
Any stars-and-bars partition $P$ with an internal separator must cut the edge $(1,k+1)$,
which means $\max_{i \in [k]} \{f(P_i)\} \ge 10k$.
Therefore, the optimal slicing of $\pi$
groups all nodes into the same block and has objective value
$k(1 - \varepsilon) + k\varepsilon = k$.
\end{proof}

\begin{remark}
Going a step further, consider graph $G$ above without edge $(1,k+1)$.
Any topological sort-based algorithm for $k$ blocks
can be as bad as $\frac{2k}{k+1} \cdot \OPT$
since any stars-and-bars partition must have an $\MTPP$ objective value
of at least $\max\{2(1-\varepsilon), (1-\varepsilon) + k\varepsilon\}$.
To see this, observe that either two heavy nodes must be in the same block
or there is a divider between each pair of adjacent heavy nodes.
Setting $\varepsilon = \frac{1}{k+1}$ proves the claim.
\end{remark}

\section{Missing details and experiments for \Cref{sec:experiments}}
\label{app:experiments}

\subsection{Partitioning algorithms}
\label{app:primal-heuristics}

We now describe the partitioning algorithms used in our experiments to compute graph partitions for the \MTPP problem:

\begin{itemize}
\item \textbf{Random~~}
This random topological sort algorithm samples $T$ i.i.d.\ weight vectors $\mat{x}^{(t)} \in [0,1]^n$ where $x^{(t)}_i \sim U(0,1)$,
maps them to topological orders $\pi^{(t)}$ as described in \Cref{sec:searching},
and returns $\min_{t \in [T]} \SliceGraph(G,\pi^{(t)}, k)$.

\item \textbf{BRKGA~~}
We run BRKGA with \texttt{BrkgaSortAndSliceDecoder} (\Cref{alg:decoder}),
and we set the population to size $100$ and the number of generations to $100$,
for $10^4$ total candidate evaluations,
which we denote as \texttt{brkga-10000} in \Cref{tab:prod_simple_approximation_ratios}.
\texttt{brkga-100} sets the population size to $10$ and the
number of generations to $10$, for $100$ total candidate evaluations.

\item \textbf{MLA~~}
A \emph{minimum linear arrangement} (MLA)
of an undirected graph $G=(V,E,w)$
is a node permutation $\pi \in \mathfrak{S}_V$
minimizing the objective
\begin{align*}
\label{eqn:mla_objective}
    h(\pi) = \sum_{\{u,v\} \in E} w(u, v) \cdot \abs*{\pi^{-1}(u) - \pi^{-1}(v)}.
\end{align*}
To generalize this idea to computation graphs,
we restrict the search by setting $h(\pi) = \infty$ if $\pi$ is not a topological order,
and setting $w(u,v) = 1$ for the (unweighted) \texttt{mla} objective
or $w(u,v) = \iocost(u, v)$ for the \texttt{mla-weighted} objective.
MLA is also an NP-hard problem, so we use BRKGA to optimize this objective
(similar to \Cref{alg:decoder}), but now the fitness is $h(\pi(\mat{x}))$.
We set the population to size $100$ and the number of generations to $100$, for a total of $10^4$ candidate evaluations.
We observed that there is a clear benefit to using \texttt{mla-weighted} over \texttt{mla} for $\MTPP$, which agrees with intuition.
Even though MLA orderings optimize for a different objective, they are still competitive, especially for being a zero-shot heuristic.
\end{itemize}

We compare the quality of these linear ordering heuristics
for \SliceGraph on the production data in \Cref{tab:prod_simple_approximation_ratios}.

\begin{table}[b!]
  \vspace{-0.5cm}
  \caption{Geometric means of approximation ratio upper bounds $\SliceGraph(G,k,\pi)/L(G,k)$ for the production models.
  Compares the \MTPP objective value induced by different linear ordering algorithms
  relative to the simple lower bound in \Cref{lem:simple_lower_bound} (lower is better).
  The $\texttt{random-}T$ heuristic generates $T$ i.i.d.\ node-weight vectors
  $\mat{x}^{(t)}$ and returns the best solution found.
  }
  \label{tab:prod_simple_approximation_ratios}
  \centering
  \vspace{0.1cm}
  \begin{tabular}{lcccccccccccccccccccccccc}
    \toprule
    Algorithm & $k=2$ & $k=4$ & $k=8$ & $k=16$ & $k=32$ & $k=64$ \\
    \midrule
    \texttt{mla} & 1.216 & 1.538 & 1.950 & 2.224 & 2.304 & 2.326 \\
    \texttt{mla-weighted} & 1.206 & 1.516 & 1.920 & 2.182 & 2.258 & 2.283 \\
    \midrule
    \texttt{random-1} & 1.220 & 1.551 & 1.959 & 2.230 & 2.309 & 2.333 \\
    \texttt{random-100} & 1.201 & 1.512 & 1.916 & 2.179 & 2.258 & 2.282 \\
    \texttt{random-10000} & \textbf{1.199} & \textbf{1.509} & 1.911 & 2.176 & 2.256 & 2.281 \\
    \midrule
    \texttt{brkga-100}  & 1.201 & 1.516 & 1.919 & 2.184 & 2.263 & 2.288 \\
    \texttt{brkga-10000}  & \textbf{1.199} & \textbf{1.509} & \textbf{1.910} & \textbf{2.175} & \textbf{2.255} & \textbf{2.272} \\
    \bottomrule
  \end{tabular}
\end{table}

\subsection{REGAL}
\label{sec:regal}

\paragraph{Dataset}
The synthetic computation graphs from \citet[Appendix A.1.4]{paliwal2019reinforced}
are constructed as follows.
The base graphs are sampled from a set of classic random graph model
(see Table 2 therein for the parameters of each random graph model).
The graphs have $50 \le n \le 200$ nodes, and are converted to directed acyclic graphs
via a random topological order.
The size of each tensor is sampled from the normal distribution
$\mathcal{N}(50, 10)$.
Each node cost is the sum of its input and output tensor
costs plus a random fraction $r$ of the total memory cost
(i.e., the sum of all tensor sizes),
where $r \sim \mathcal{N}(0, 0.1)$.
If a node has more than one output tensor, we use the lexicographically
least according to the tensor index.
Finally, \citet{paliwal2019reinforced} filter these graphs and only keep
those that are sufficiently hard for their min-peak scheduling objective.

\paragraph{Results}
We report a parallel set of results on the REGAL dataset
(see \Cref{fig:regal_running_times} and \Cref{tab:regal_scaled_lower_bounds})
to accompany our running time plots and table of lower bound ratios
for the production models in \Cref{sec:experiments}.

\begin{table}[H]
  \vspace{-0.5cm}
  \caption{Geometric means of the best available lower bound from the MIP hierarchy,
  normalized by the best solution found using BKRGA, across the REGAL dataset.}
  \label{tab:regal_scaled_lower_bounds}
  \centering
  \vspace{0.2cm}
  \begin{tabular}{lcccccccccccccccc}
    \toprule
    Lower bound & $k=2$ & $k=4$ & $k=8$ & $k=16$ & $k=32$ & $k=64$ \\
    \midrule
    \simple (\Cref{lem:simple_lower_bound}) & 0.9579 & 0.8795 & 0.7911 & 0.5821 & 0.4087 & 0.3336 \\
    \bottleneck & 0.9794 & 0.8818 & 0.7918 & 0.5824 & 0.4090 & 0.3338 \\
    \bottleguess & 0.9804 & 0.8826 & 0.7918 & 0.5824 & 0.4090 & 0.3338 \\
    \exact & 0.9804 & 0.9579 & 0.9407 & 0.8929 & 0.5910 & 0.3810 \\
    \bottomrule
  \end{tabular}
\end{table}

\begin{figure*}
\centering
\includegraphics[width=1.0\textwidth]{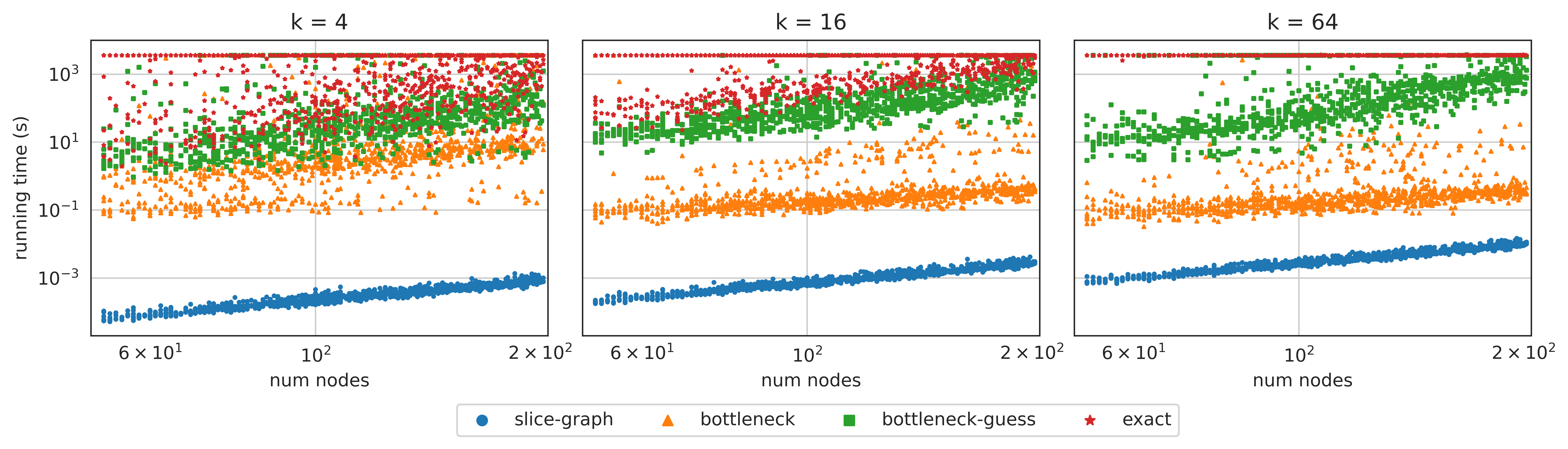}
\vspace{-0.75cm}
\caption{Running times of $\SliceGraph$ and different MIP lower bound computations across the REGAL models. Each point denotes a run for one graph, color-coded to denote $\SliceGraph$ partitioning vs.\ \bottleneck, \bottleguess, and \exact lower bounds. The \bottleguess times are summed across all $k$ MIP instances involved. Each plot is for a different value of $k$.
In order to facilitate visual comparisons across the plots, all three employ the same $y$-axis.
Some of the data tops out at 3600 seconds
since that is where we set the MIP time limit.
}
\vspace{-0.3cm}
\label{fig:regal_running_times}
\end{figure*}

\subsection{Solving the MIPs}
\label{app:mip_solving}

To solve the MIPs that underpin \Cref{tab:prod_scaled_lower_bounds} and \Cref{tab:regal_scaled_lower_bounds}, we used a combination of the Gurobi~\citep{gurobi} and SCIP~\citep{BestuzhevaEtal2021OO} solvers. 
For \texttt{bottleneck}, we used Gurobi with a 15-minute time limit,
and for \texttt{exact}, we relaxed this to 60 minutes because these MIPs are tougher to solve.
For technical reasons having to do with our computing setup, we instead used SCIP for \bottleguess, with a 60-minute time budget that was shared across all values of $k$ of the MIPs that compose a single \bottleguess instance (one MIP per guess). A non-trivial fraction of the instances failed to solve to provable optimality within the time limit, especially for \texttt{exact} with $k = 64$. In these cases, the solver still returns a valid lower bound, and we use that.

\end{document}